\documentclass[11pt]{article} 

\usepackage[utf8]{inputenc} 
\usepackage[reqno]{amsmath}
\usepackage{amsthm}
\usepackage{graphicx}
\usepackage{epstopdf}
\usepackage{amssymb}
\usepackage{tikz}
\usepackage{tikz-qtree}
\usepackage{float}

\usepackage{mathtools}
\mathtoolsset{showonlyrefs}

\newtheorem{definition}{Definition}
\newtheorem{proposition}{Proposition}

\usepackage{geometry} 
\usepackage{graphicx} 

\usepackage{booktabs} 
\usepackage{array} 
\usepackage{paralist} 
\usepackage{verbatim} 
\usepackage{subfig} 

\usepackage{fancyhdr} 
\usepackage{sectsty}
\allsectionsfont{\sffamily\mdseries\upshape} 

\usepackage[nottoc,notlof,notlot]{tocbibind} 
\usepackage[titles,subfigure]{tocloft} 

\newcommand{\eps}{\varepsilon}
\DeclareMathOperator*{\trace}{trace}

\title{Causal Interfaces}
\author{David A. Eubanks}

\begin{document}
\maketitle
\begin{abstract}
The interaction of two binary variables, assumed to be empirical observations, has three degrees of freedom when expressed as a matrix of frequencies. Usually, the size of causal influence of one variable on the other is calculated as a single value, as increase in recovery rate for a medical treatment, for example. We examine what is lost in this simplification, and propose using two interface constants to represent positive and negative implications separately. Given certain assumptions about non-causal outcomes, the set of resulting epistemologies is a continuum. We derive a variety of particular measures and contrast them with the one-dimensional index.
\end{abstract}

\section{Introduction}
Everyday life depends on assumptions about cause and effect. Flipping a switch causes a light to come on, and  turning a key causes an automobile's engine to start. A modern account of causality and its relationship to mathematical probability and statistics can be found in the work of Judea Pearl, including \cite{pearl2009causal} and \cite{pearl2000causality}, and his development of structural analysis in \cite{pearl2009causality} provides a logical framework for causal reasoning. The aim of the present work is to focus on the simplest-possible causal connections: the summary of an experimental interaction between two binary variables. Following Pearl's distinction between purely observational data and experimental data, we will generally assume that a causal variable $A$ was manipulated randomly, with $do(A=1)$ or $do(A=0)$, in order to see what effect it has $B$. If we are directly causing $A=0$ and $A=1$ using a random assignment, this allows us to rule out (with probability approaching one) a third confounding variable $C$ from simultaneously controlling $A$ and $B$ to give the false appearance of causality.

Assessing the impact of $A$ on $B$ can be done in many ways, including probability-raising and odds ratios. An attempt to catalog and compare these can be found in \cite{fitelson2011probabilistic}. Probability-raising is intuitive: the magnitude calculated from is the probability $\Pr{(B|A)}$ versus either $\Pr{(B)}$ or $\Pr{(B|A')}$ \cite{sep-causation-probabilistic}, where the prime means complement or logical negation. As an example: can a neighborhood baseball game cause a window of a house bordering the outfield to break? Were we to conduct an experiment to determine the answer, we would randomly assign baseball games to fields in various neighborhoods, and measure the breakage of windows relative to a control group of houses that bordered baseball-less fields. If we saw an increase in the percentage of fractured panes in the with-baseball treatment group, we might conclude that this additional misfortune was caused by the games nearby. An ``increase in probability" as a marker for causality is reasonable, but not entirely satisfactory, because it relies on a background rate of window-breakage. In neighborhoods without glass windows, it would be undetectable, for example, even though the causal power of baseballs over window panes would presumably be unaltered. This example helps distinguish between the statistical context that an experiment is performed in and the presumed causal effect that we are interested in. The former is likely to change unless we are dealing with fundamental properties of the universe. Electrons are interchangeable as units of experimentation, but houses and their windows are all different.

The approach we take here is that the assertion of causality based on observations is a conversion from observational data, which is statistical in nature, into language-based assertions, which are logical in nature. The intention is to inform discrete decisions like ``take the medicine or not'' rather than ``what percentage of the pill should I chop off and swallow?''. We therefore take causal inferences to be specific and universal, but subject to masking by other causes. We might say that ``light switches always turn lights on'', and then account for the exceptions as part of the particular context. Burned out bulbs, power outages, malfunctioning switches, bad wiring, and so on, are cases of the causality being masked by context, not violating the causality itself. In the present work we will use the word ``confusion'' to denote contextual interference with causes, which happens with some estimated probability. The estimate will only be valid for the context in which it was gathered.

If causality is treated as a strict logical implication $A \Rightarrow B$, then exceptions have to be accounted for in other factors. If $A$ increases the probability of $B$, this is an indication of a cause at work, but the probability increase is not the cause itself. In the example, the baseball game probably does increase the probability of a broken window, but it is the baseball itself that does the breaking. In this way we parse the situation as the game increasing the probability of a causal event, but if it does occur the implication from [baseball hits window] to [window breaks] is certain. The probability increase depends on context, but the causal implication does not. The former depends on a list of sufficient conditions (for windows breaking) that were present in the experimental samples, whereas the latter is itself a sufficient condition. In this case, other sufficient conditions would include sonic booms, giant hail, a rock thrown by a lawn mower, an accident inside the house, a large bird flying into the window, and large explosions nearby. Some of these will be rare, and it is unlikely that our experiment will sample them all, and even in the unfortunate circumstance that we did sample them all, they are not likely to recur with stable frequencies as the world evolves. We can think of all these sufficient causes as binary random variables that are constructed by the unfolding of time. They may be stochastic with predictable regularity or they may better thought of as the result of a very complex computing machine that occasionally produces a mark on a tape. The former context can be treated well with parametric statistics. For example, mean and variation can help understand crop yield under controlled cultivation. On the other hand, the hallmark of intelligence (or just computation) is the chaining together of unlikely events, so that a given sequence of cause-and-effect may only happen once. As an example, the internal state of a given desktop computer be unique in the past and future history of the universe, and it is not best understood as random but as the consequence of a large number of deterministic logical operations, which has severe logical limits (the Halting Problem and its kin) \cite{turing1936computable}.

If $A$ is sufficient for $B$, then evidence for other sufficient conditions comes to use through $B \Rightarrow A$. In the earlier example, this converse is [window breaks] implies [baseball game], which we use in the contrapositive form [no baseball game] implies [window doesn't break] because the variable under experimental control is the presence of the game. We will call this implication ``negative causality''. This term is loaded with philosophical connotations that we will not delve into \cite{birnbacher2012negative}, but instead just use it as a convenient name for the converse of ``positive causality''. If we only observe windows breaking in the presence of baseball games, then the contexts of those experiments may exclude other sufficient causes. Negative causality is of interest because lack of it can mask evidence for positive causality. For example, if the baseball games are played on the coast, and all the windows are hurricane-proof, this fact will mask the causality of baseballs breaking normal windows. The associated logic would be [baseball hits window] $\Rightarrow$ [window breaks] OR [hurricane-proof windows] OR ... . Low negative causality points to other sufficient causes, whereas a value of unity means that the context of the experiment did not show evidence of any other causes of $B$, leaving us to conclude that (in context) $A \Leftrightarrow B.$

Binary relationships of the type under consideration here can be considered categorical, and a comprehensive introduction to the analysis of categorical variables can be found in \cite{agresti2007introduction}; a number of ways of calculating the influence of $A$ on $B$ can be found there. The $A \times B$ interaction tables are also put to use in the classification problem of $A$ predicting $B$, which differentiates true positive and negative from false positive and negative cases for (often) purely observational data \cite{kolo2011binary}. Although we have to give up causality, the techniques described in the following sections can also be used for observational data. See in particular Section \ref{ss:C}.

For simplicity, we will assume that the number of observations is large, and work directly with frequencies instead of counts. We organize these into a matrix
\[
P
=
\begin{bmatrix}
  p_{00} & p_{01} \\
  p_{10} & p_{11}
 \end{bmatrix}
\]
which sums to one, and where the row index is the value of $A$ and the column index is the value of $B$. We will usually assume that we have experimental control over $A$, and that for each observation that provides information for $P$, the observation of $B$ follows the observation of $A$ in time in some reasonable way. It will be convenient to refer to column and row sums of the $P$ matrix in a compact way, and we will use a slightly modified form of the notation in \cite{agresti2007introduction}, where the sum of the two $i$ row elements will be denoted $p_{i*}$ and the sum over column $j$ as $p_{*j}$.

As an example, we imagine a drug test with a trial group $A=1$ and a randomly-selected control group $A=0$, testing to see if patients recover within a certain period ($B=1$) or fail to recover in that period ($B=0$). Assume that the experiment yields the frequencies
\[
P_1 =
\begin{bmatrix}
  .23 & .25 \\
.20  &  .32
 \end{bmatrix}.
\]

What should we conclude about the effect of the treatment? The matrix has four elements, but because they sum to one, there are three degrees of freedom. Therefore, a single index of the effect of the drug therefore cannot tell us everything about the drug's effect and the content in which it happens. Nevertheless, we might begin with the difference between the success rate with treatment and the success rate without, which we will denote as $\hat{\eps}$ throughout, as a point of reference for other measures we will derive. We have
\begin{align}
\hat{\eps} &= \frac{p_{11}} {p_{10} + p_{11}} - \frac{p_{01}}{p_{00} + p_{01}}  \\
&= \frac{p_{11}} {p_{1*}} - \frac{p_{01}}{p_{0*}} \label{eq:epsilon0} \\
&= \frac{p_{00}p_{11}-p_{01}p_{10}}{(p_{00}+p_{01})(p_{10}+p_{11})} \\
&= \frac{p_{00}p_{11}-p_{01}p_{10}}{p_{0*}p_{1*}} \label{eq:epsilon}
\end{align}
It is easy to show that $\hat{\eps} \epsilon [-1,1]$. It is easy to show that the numerator is the determinant of the matrix, which is also the covariance of $A,B$. The denominator is the variance of $A$.  We have
\begin{equation}
\hat{\eps} = \frac{Cov(A,B)}{Var(A)}.
\end{equation}
A couple of other connections are worth mentioning. If we view $B$ as a predictor of $A$ (zero predicts zero and one predicts one), then the area under the true positive versus false positive curve (the so-called ROC curve) is $1/2+\hat{\eps}/2$. Additionally, $\hat{\eps}$ is the slope of the regression line for $B$ as a predictor of $A$.

In the example, $\hat{\eps} \approx .10$ meaning that $10\%$ more subjects recovered with treatment than without. It can be the case that the effect is negative, which we can see by reversing the columns of the matrix to get $\begin{bmatrix}  .23 & .25 \\ .20  &  .32  \end{bmatrix}$, with associated $\hat{\eps} \approx -.10$. This would be the case if the treatment did actual harm by reducing the chance of recovery. Rather than using negatives to quantify effects, we will choose to think of this as a semantic problem and avoid it in the following way. Without loss of generality, we can change a negative $\hat{\eps}$ into a positive one by inverting the output designations, swapping the meaning of $B=0$ and $B=1$. As an example, a vaccine's effect is intended to be negative in the logical sense: it prevents a condition rather than causing one. In this case, we could set $A=1$ for vaccinated subjects, $A=0$ for the others, and then use $B=0$ to categorize subjects who contracted the disease the vaccine is intended to prevent, with $B=1$ the remainder (those who did not contract it). Assuming that the vaccine actually does have the desired effect, we will have $\hat{\eps}>0$. Mathematically, this is a permutation matrix on the columns (i.e., post-multiplied). We will consider any such transformation already accomplished, so that we can assume the the interaction matrix $P$ always has $p_{00}p_{11}-p_{01}p_{10} \geq 0$. When we want to distinguish these cases of ``backwards'' variables from the ``forwards'' ones, we will call them \emph{anti-causal}.

Next we examine a weakness of $\hat{\eps}$ as an measure of influence of $A$ over $B$. The matrix
\[
P_2 =
\begin{bmatrix}
  .05 & .45 \\
  0  &  .50
 \end{bmatrix}
\]
has $\hat{\eps} \approx .10,$ about the same as the previous example. The two situations represented by $P_1$ and $P_2$ are qualitatively different, however. If we were just looking at the bottom rows of these matrices (the cases where the treatment was applied), we would consider $P_2$ to be superior to $P_1$, because 100\% of the treatments seemed to be successful in the former. This detail is lost when we compress the effectiveness of the treatment into a single index. We might try to create a complementary second index that measures the effect of \emph{non-treatment}.  A familiar example will underscore this point.

Residential light switches are designed to be simple interfaces: one position (``up'') turns the light on, and the other (``down'') turns it off. If we built an interaction table like $P_1$ and $P_2$ based on observations of switch and light, it might look like
\[
P_3 =
\begin{bmatrix}
  .45 & 0 \\
  .02  &  .53
 \end{bmatrix},
\]
where $p_{10}$ represents the frequency that the switch is up, but the light remains off. This could be because of a burned-out bulb or power outage. A simple model of the light's state as influenced by the switch, bulb, and power state can be written in Boolean algebra as
\[
L = S B P.
\]
Negating this gives us the equivalent statement about $L'$,
\[
L' = S' + B' +  P'.
\]
The off state of $L$ is controlled by an OR relationship to the down position of $S$, whereas the on state is complicated by influences beyond the control of the switch. In this sense, light switches are better at turning lights off than turning them on in this model because $S' \Rightarrow L'$ but $S \nRightarrow L$. Other asymmetrical causes can easily be found. Consuming poison causes one to die, but abstaining from it doesn't cause one to live (otherwise ambulance crews would carry around vials of the stuff so they could save accident victims by not giving them poison). On the other hand, there are causal relationships that work both directions, and they are valuable for understanding and creating the world. Communications systems are built on if-and-only-if relationships between signal and intent, like the tap of a telegraph key and the resulting ``dit'' on the other end of the wire.

In order to capture asymmetrical causality we will look for positive effects and negative effects as separate measures. This amounts to two questions: ``How much control does the switch have over turning the light off?'' and ``How much control does the switch have over turning the light on?'' We will refer to these two parameters together as an \emph{interface}, with two \emph{interface coefficients} $(\eps_0,\eps_1)$ that measure these respective types of control.

\section{Interfaces and Indexes}

An elevator is designed to be an interface that allows us to get to some desired floor in a building. For the sake of illustration, let our input $A=1$ be the state of our pressing the button numbered with our floor, and $B=1$ the state of exiting the elevator car onto that floor within one minute after entering the elevator. If we are alone in the building, the elevator acts like a perfect interface; we will get to out floor (quickly) if and only if we press the button. A matrix of experimental frequencies would be diagonal. In this case, we have full control, and it is natural to expect that $(\eps_0,\eps_1) = (1,1).$

If the building is occupied with other travelers, there may be occasions when our ride is slowed appreciably, so that we experience instances of $B=0$ (not arriving within one minute) even though we pressed the button ($do(A=1)$). On the other hand, if we don't press the button ($do(A=0)$), there is still a chance that the car will stop on our floor because either other passengers press our button, or someone on that floor signals a stop. After repeated trips up and down, careful observation might reveal an interaction matrix like $\begin{bmatrix}.3 & .2 \\ .05 & .45 \end{bmatrix}.$ We would expect in this case that the positive effect of pressing the button is diminished but still strong, but the negative effect of refraining is weakening faster. $\eps_0$ and $\eps_1$ are both less than one now.

Finally, imagine that a convention on parenting has arrived at the hotel, bringing along their children who have come for behavioral training. Some of these young people like to amuse themselves by pressing all the buttons on the elevator. This happens very frequently when we ride, slowing our ascent appreciably. Now the interface is broken (our act of button-pressing has no positive or negative control over the outcome of arriving within a minute), and the matrix could be $\begin{bmatrix}.25 & .25 \\ .25 & .25 \end{bmatrix},$ in which case we would expect that $(\eps_0,\eps_1) = (0,0).$

Informally, we are thinking that the interface comprises
\begin{align}
&\eps_0: \text{the contextual probability of }  A' \Rightarrow  B' \\
&\eps_1: \text{the contextual probability of } A \Rightarrow B.
\end{align}
A first approach might be to use $\hat{\eps}$ to define a complementary measure of negative causation $\check{\eps}$, with
\begin{align}
\check{\eps} &= \frac{p_{00}} {p_{00} + p_{01}} - \frac{p_{10}}{p_{10} + p_{11}} \\
&= \frac{p_{00}p_{11}-p_{01}p_{10}}{(p_{00}+p_{01})(p_{10}+p_{11})}  \\
&= \hat{\eps}.
\end{align}
This shows  that $\hat{\eps}$ is a symmetric measure of causality, not distinguishing between positive and negative cause. To differentiate these, we can begin by decomposing the original matrix into something like signal and noise. If an interface comprises a diagonal matrix of positive and negative cause, then subtracting this from the original matrix must leave the observations not explained by the interface, which we will call the \emph{confusion matrix} $C$. One approach might be to decompose example $P_1$ into
\[
\begin{bmatrix}
  .23 & .25 \\
.20  &  .32
 \end{bmatrix}
=
\begin{bmatrix}
  0 & .25 \\
.20  &  0
 \end{bmatrix}
+
\begin{bmatrix}
  .23 & 0 \\
0  &  .32
 \end{bmatrix},
\]
with the confusion matrix being the left summand and the interface the right. This entails a particular epistemological choice, wherein we assume that all cases can be explained by a causal variable $A$ and an unknown anti-causal variable. In terms of the drug test scenario, it means that we consider all the treatment cases that recovered to be the sole effect of the drug, and all the cases where the patient did not take the drug and did not recover to be the sole effect of withholding treatment. This stance means that we take the effectiveness of the administration of the drug ($do(A)$) to behave like an indicator for some third variable. For example, if the drug acts like a perfect interface for subjects with a particular genetic condition and has the reverse effect on anyone else. Then the effectiveness of the drug on a case-by-case basis gives the same information as a genetic test for this condition.

Anti-causal subsets always exist if the subjects are distinguishable one from another. We can trivially create such a subset \textit{post facto} by listing suitable observations. The existence of a generalizable index that points to an anti-causal subset \emph{before} the experiment is more interesting.

As an example, a light switch may produce an interaction with a ceiling light like
$\begin{bmatrix}
  .25 & .25 \\
.25  &  .25
 \end{bmatrix}$, from which we may conclude that the switch has no influence over the light at all. Perhaps it is the wrong switch for that light. However, if we learn of a second switch wired as a ``two-way'' arrangement (often found at the top and bottom of a stairway), the decomposition
\[
\begin{bmatrix}
  .25 & .25 \\
.25  &  .25
 \end{bmatrix}
=
\begin{bmatrix}
  0 & .25 \\
.25  &  0
 \end{bmatrix}
+
\begin{bmatrix}
  .25 & 0 \\
0  &  .25
 \end{bmatrix}
\]
makes perfect sense as a union of anti-causal and causal subsets. Two experimenters, unaware of the wiring or each other, could simultaneously gather switch/light data and conclude that these were independent of one another. The $P$ matrix by itself does not contain enough information to identify anti-causal indexes.

We turn now to a weaker assumption about confusion. Rather than assigning maximum power to the putative cause $A$, we will assume that some portion of the observations is distributed independently of $A$.

\section{Preliminary Confusion}

Because we already have a workable idea of an interface using the symmetric $\hat{\eps}$, we can subtract that effect from the original matrix and see what ``confusion'' looks like in this case. We use the following relationship to define the confusion matrix $C$:
\begin{equation}
\begin{bmatrix}
  p_{00} & p_{01} \\
  p_{10} & p_{11}
 \end{bmatrix}
=
(1-\hat{\eps})
\begin{bmatrix}
  c_{00} & c_{01} \\
  c_{10} & c_{11}
 \end{bmatrix}
+
\hat{\eps}
\begin{bmatrix}
  p_{0*} & 0 \\
0  &  p_{1*}
 \end{bmatrix}.
\end{equation}
 The idea is that the coefficient $\hat{\eps}$ represents the amount of causal control $A$ exerts over $B$, and subtracting the appropriate diagonal matrix leaves the left-over ``independent confusion''. This decomposition has some nice properties.

\begin{proposition} For  $\hat{\eps} > 0$, the confusion matrix $C$ is a frequency matrix.
\end{proposition}
\begin{proof}[Proof]
We have to show that the elements of $C$ are in $[0,1]$ and that they sum to one. The second part is easy, since we are using $\hat{\eps}$ to create a convex combination that sums to the original matrix $P$. Since both $P$ and $\begin{bmatrix}  p_{0*} & 0 \\ 0  &  p_{1*} \end{bmatrix}$ sum to one, $C$ must also. For the rest, observe that
\begin{align}
C &=\frac{1}{1-\hat{\eps}}
\begin{bmatrix}
    p_{00}-\hat{\eps} p_{0*} & p_{01} \\
   p_{10}  &  p_{11}-\hat{\eps} p_{1*}
 \end{bmatrix}  \\  \label{eq:confusion}
&= \frac{1}{1-\hat{\eps}}
\begin{bmatrix}
  p_{0*} \left( 1-\frac{p_{11}}{p_{1*}} \right)  &  p_{01} \\
   p_{10}  &       \frac{p_{1*}p_{01}}{p_{0*}}
 \end{bmatrix},
\end{align}
where the second line is obtained by substituting the definition of $\hat{\eps}$ from $\eqref{eq:epsilon0}$. The matrix entries are all obviously non-negative. Since the whole matrix sums to one, no single element can be more than one.
\end{proof}

\begin{proposition} The confusion matrix $C$ has determinant zero.
\end{proposition}
\begin{proof}[Proof]
 Using $\eqref{eq:confusion}$, we need to show that $ p_{01}p_{10} =   p_{0*}(1-\frac{p_{11}}{p_{1*}}) \frac{p_{1*}p_{01}}{p_{0*}}$. Expanding the right hand side gives us
\begin{align}
p_{0*} \left( 1-\frac{p_{11}}{p_{1*}} \right) \frac{p_{1*}p_{01}}{p_{0*}} &=(p_{1*}-p_{11})p_{01} \\
&=  p_{01}p_{10}  & \text{ as desired.}
\end{align}
\end{proof}
Another way to say that $C$ has determinant zero is to say that one row (or column) is a multiple of the other. Because it is a probability matrix, this also means that it can be decomposed into an outer product of two distributions
\[ C = \begin{bmatrix}
    \sigma_{0*}  \\
   \sigma_{1*}
 \end{bmatrix}  \\
\begin{bmatrix}
    \sigma_{*0}  &   \sigma_{*1}
 \end{bmatrix},
\]
where $\sigma$ is a row or column sum, following our notation. An interaction matrix that is the product of its marginal distributions is what we would expect from two independent random variables. Therefore, we can interpret $C$ as the frequency of interaction between $A$ and $B$ that is outside the control of the interface, which seems appropriate for a confusion matrix. We will use this idea to generalize to two interface constants, and decompose $P$ as
\begin{equation}
  \begin{bmatrix}
  p_{00} & p_{01} \\
  p_{10} & p_{11}
 \end{bmatrix}
=
\begin{bmatrix}
1-  \eps_0 & 0 \\
  0 &  1- \eps_1
 \end{bmatrix}
 \begin{bmatrix}
  c_{00} & c_{01} \\
  c_{10} & c_{11}
 \end{bmatrix}
+
\begin{bmatrix}
 \eps_0 & 0 \\
  0 &  \eps_1
 \end{bmatrix}
\begin{bmatrix}
  p_{0*} & 0 \\
0  &  p_{1*}
 \end{bmatrix}, \label{eq:interface}
\end{equation}
with $\eps_0$ as the interface constant that represents the effect of withholding treatment (or whatever $A=0$ can be interpreted as), and $\eps_1$ the positive effect of $A=1$ on $B=1$. We will require that the $C$ matrix be the product of its marginals, so that it can be thought of as a distribution of two independent random variables. The interface coefficients we interpret as the probability of a causal event (positive or negative). The multiplication of the $\eps$s by row sums $p_{0*}$ and $p_{1*}$ make these coefficients invariant under row-scaling. Because this is considered to be experimental data, we control the relative frequency of $A=1$ and $A=0$, and the measure of causality should be independent of that.

The scaling effect of $1-\eps$ and $\eps$ on the matrix decomposition of $P$ can be thought of as a uniform random variable $U$ that ``decides'' whether causality due to our intervention, rather than confusion from some masking effect, determines the outcome. When we apply the cause $A=0$, if $U > \eps_0$, then the causal implication makes $B=0$, and otherwise, we take our chances with the confusion matrix (which may still end up with $B=0$ from other causes). Together these conditional interactions account for all the observations found in $P$, and the theoretical interest is in systematically subtracting out the interface (causal) effects from what we might consider environmental noise.

So far, we only have one measure, $\hat{\eps}$ of causality. Next we show that there is a continuum of possible measures.

\section{More Varieties of Confusion}
The condition that $C$ be a product of its margins can be used to derive a relationship for the interface coefficients, since
\[
\begin{bmatrix}
1-  \eps_0 & 0 \\
  0 &  1- \eps_1
 \end{bmatrix}
 \begin{bmatrix}
  c_{00} & c_{01} \\
  c_{10} & c_{11}
 \end{bmatrix}
=
  \begin{bmatrix}
  p_{00} - \eps_0 p_{0*} & p_{01} \\
  p_{10} & p_{11} - \eps_1 p_{1*}
 \end{bmatrix},
\]
but the row-scaling of $C$ won't affect its singularity, so we want the determinant of the right side to be zero, or
\begin{equation}
(p_{00} - \eps_0 p_{0*})(p_{11} - \eps_1 p_{1*}) =  p_{01} p_{10} \label{eq:detzero}
\end{equation}
We can write $\eps_1$ as a function of $\eps_0$ with
\begin{equation}
\eps_1= \frac{1}{p_{1*}} \left( p_{11}-\frac{p_{01} p_{10}}{p_{00}-\eps_0 p_{0*} } \right), \label{eq:efunction}
\end{equation}
which is a hyperbola (unless $p_{01}p_{10}=0$, in which case it is vertical and horizontal line segments) with its lower half having $x$-intercept and $y$-intercept given by
\begin{align}
x_0 &= \frac{p_{00}p_{11}-p_{01}p_{10}}{p_{0*}p_{11}} \\
&= \frac{Cov(A,B)}{p_{0*}p_{11}}  \\
y_0 &= \frac{p_{00}p_{11}-p_{01}p_{10}}{p_{1*}p_{00}} \\
 &= \frac{Cov(A,B)}{p_{1*}p_{00}}.
\end{align}

 The upper half of the hyperbola passes through $(1,1)$, although that is only a valid solution when the original matrix is diagonal.  This is, however, convenient, since it rules out any other solutions on the upper curve of the hyperbola. Note that the bounds shown above do not limit the interface coefficients from attaining all the values on the lower half of the hyperbola in the first quadrant, framed by the intercepts. This gives us a neat picture of all possible independent-confusion epistemologies. Two examples are shown in Figure \ref{fig:An interface}.

\begin{figure}[H]
  \centering
  \includegraphics[width=0.6\textwidth]{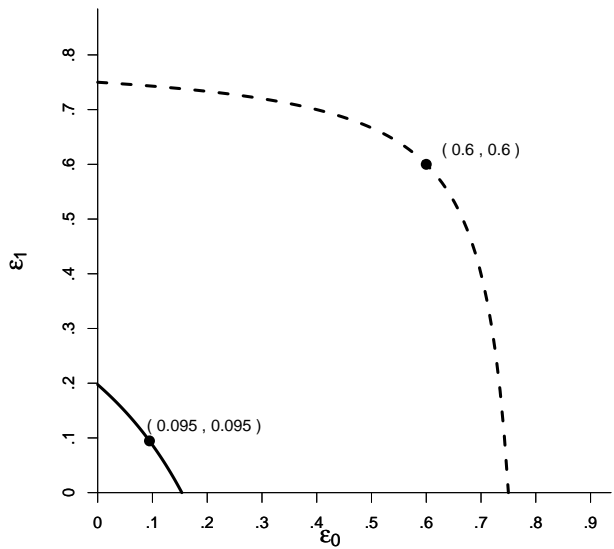}\\
  \caption[Causal coefficients $\eps_0$ versus $\eps_1$]{Interface plots for  $\begin{bmatrix}
.23  & .25 \\
  .20 & .32
 \end{bmatrix}$ (solid) and $\begin{bmatrix}
.40 & .10 \\
.10 & .40
 \end{bmatrix}$ (dashed), with $\hat{\eps}$ identified.}\label{fig:An interface}
\end{figure}

The relationship in \eqref{eq:efunction} allows us to think of a causal relationship between binary $A$ and $B$ as a continuum that we choose from based on some philosophical principle.

\begin{definition}An interface $\emph{epistemology}$ is method of uniquely assigning interface coefficients $(\eps_0, \eps_1)$, so that $\eqref{eq:interface}$ holds when $p_{00}p_{11}-p_{01}p_{10} \geq 0$ .
\end{definition}

Next we do some calculations that will simplify working with independent-confusion interfaces.

\section{Normalizing $P$}
The decomposition $\eqref{eq:interface}$ leaves one degree of freedom in identifying a particular $(\eps_0,\eps_1)$ for a given $P$. In this section we show that this amounts to the choice of the column distribution of $C$, and derive some consequences of that fact.

It is natural to assume that $C$ is of the form
\begin{equation}
C = \begin{bmatrix}
\sigma_0 p_{0*} & \sigma_1 p_{0*} \\
\sigma_0 p_{1*} & \sigma_1 p_{1*}
\end{bmatrix},
\end{equation}
where $(\sigma_0,\sigma_1)$ forms a distribution. Notice that in this instance we have
\begin{align}
 \begin{bmatrix}
  p_{00}/p_{0*} & p_{01}/p_{0*} \\
  p_{10}/p_{1*} & p_{11}/p_{1*}
 \end{bmatrix}
=&
\begin{bmatrix}
(1- \eps_0)\sigma_0 +\eps_0 & (1- \eps_0)\sigma_1 \\
(1- \eps_1)\sigma_0 & (1- \eps_1)\sigma_1 +\eps_1
 \end{bmatrix} \\
=& \begin{bmatrix}
\eps_0 \sigma_1+\sigma_0 & \sigma_1 - \eps_0 \sigma_1 \\
\sigma_0 - \eps_1 \sigma_0 & \eps_1 \sigma_0+\sigma_1
 \end{bmatrix},
\end{align}
from which we can use the first row to get the system
\begin{align}
p_{00}/p_{0*} =& \eps_0 \sigma_1 + \sigma_0 \\
p_{01}/p_{0*} =& \sigma_1 - \eps_0 \sigma_1.
\end{align}
Adding these two equations gives
\begin{align}
\frac{p_{00}+p_{01}}{p_{0*}} =&  \sigma_1 + \sigma_0 \\
1 &= 1.
\end{align}
A similar calculation works for the second row, showing that any distribution $(\sigma_0,\sigma_1)$ leads to a self-consistent definition of the interface coefficients, with
\begin{align}
\eps_0 =& 1- \frac{p_{01}}{p_{0*}\sigma_1} \\
\eps_1 =& 1- \frac{p_{10}}{p_{1*}\sigma_0}.
\end{align}

This result can be used in the other direction as well, to find the useful lower limits of $(\sigma_0,\sigma_1)$, at the intercepts in the $\epsilon_0$-$\epsilon_1$ first quadrant.
\begin{align}
0 =& 1- \frac{p_{01}}{p_{0*}\sigma_1} \text{implies that} \\
\sigma_0 =& \frac{p_{01}}{p_{0*}}, \text{and} \\
0 =& 1- \frac{p_{10}}{p_{1*}\sigma_0} \text{implies that} \\
\sigma_1 =& \frac{p_{10}}{p_{1*}}.
\end{align}

This result demonstrates that the interface decomposition in $\eqref{eq:interface}$ only depends on the row-normalized version of $P$. We can rewrite the relationship between interface coefficients in  $\eqref{eq:epsilon0}$ as
\begin{equation}
\eps_1 = \frac{p_{11}}{p_{1*}} - \frac{p_{10}}{p_{1*}} \left( \frac{p_{01}/p_{0*}}{p_{00}/p_{0*}-\eps_0} \right) \label{eq:newepsilon0}.
\end{equation}

Generally, these are called right-stochastic matrices, and are often associated with Markov Chains. We can now write a simpler version of the interface decomposition, using a new row-normalized matrix of observations, as
\begin{equation}
\begin{bmatrix}
r_{00} & r_{01} \\
r_{10} & r_{11}
\end{bmatrix} := \begin{bmatrix}
p_{00}/p_{0*} & p_{01}/p_{0*} \\
p_{10}/p_{1*} & p_{11}/p_{1*}
\end{bmatrix},
\end{equation}
so that
\begin{equation}
  \begin{bmatrix}
  r_{00} & r_{01} \\
  r_{10} & r_{11}
 \end{bmatrix}
=
\begin{bmatrix}
1-  \eps_0 & 0 \\
  0 &  1- \eps_1
 \end{bmatrix}
 \begin{bmatrix}
  \sigma_0 & \sigma_1 \\
  \sigma_0 & \sigma_1
 \end{bmatrix}
+
\begin{bmatrix}
 \eps_0 & 0 \\
  0 &  \eps_1
 \end{bmatrix}.  \label{eq:newinterface}
\end{equation}
This can also be written in outer-product form as
\begin{equation}
  \begin{bmatrix}
  r_{00} & r_{01} \\
  r_{10} & r_{11}
 \end{bmatrix}
=
\begin{bmatrix}
1-  \eps_0 \\
1- \eps_1
 \end{bmatrix}
 \begin{bmatrix}
  \sigma_0 & \sigma_1
 \end{bmatrix}
+
\begin{bmatrix}
 \eps_0 & 0 \\
  0 &  \eps_1
 \end{bmatrix}.  \label{eq:newinterfaceOuter}
\end{equation}

The interface relationship becomes
\begin{equation}
\eps_1 = r_{11} - \frac{r_{10}r_{01}}{r_{00}-\eps_0}, \label{eq:efunctionr}
\end{equation}
with
\begin{align}
\eps_0 \epsilon& \left[0, \frac{r_{00}r_{11}-r_{01}r_{10}}{r_{11}} \right] \label{eq:xintercept} \\
\eps_1 \epsilon& \left[0, \frac{r_{00}r_{11}-r_{01}r_{10}}{r_{00}} \right] \label{eq:yintercept}
\end{align}
(note that the numerators can be written in other forms, such as the surprisingly simple $r_{00}-r_{10}$),
derivable from
\begin{align}
\eps_0 =& 1 - \frac{r_{01}}{\sigma_1} \label{eq:epsilonandsigma0} \\
\eps_1 =& 1 - \frac{r_{10}}{\sigma_0} \label{eq:epsilonandsigma1},
\end{align}
where the distribution limits (always subject to $\sigma_0 + \sigma_1 = 1$) are
\begin{align}
\sigma_0 \epsilon& [r_{10},r_{00}] \\
\sigma_1 \epsilon& [r_{01},r_{11}].
\end{align}
It will also be handy to have the inverse relationships for \eqref{eq:epsilonandsigma0} and \eqref{eq:epsilonandsigma1} at hand:
\begin{align}
\sigma_0 =& \frac{r_{10}}{1-\eps_1} \\
\sigma_1 =& \frac{r_{01}}{1-\eps_0}.
\end{align}

 The decomposition \eqref{eq:newinterface} comprises a bijection from $R$ into a set of hyperbolas with the intercepts given in \eqref{eq:xintercept} and \eqref{eq:yintercept}. The curve is therefore uniquely identified with the given $R$, but we are after a \emph{property} of $R$ (causality with $B$). This depends on what we believe about the distribution of $B$ that is not caused by $A$. Given a choice for the confusion distribution $\sigma_1$, we can identify particular values for the interface coefficients.

\section{A Selection of Epistemologies}
In this section we will look at several choices for $(\sigma_0,\sigma_1)$ and the resulting interfaces. We can proceed by first choosing ($\eps_0,\eps_1)$ and then calculating the confusion matrix, or in the other direction, by assuming something about the distribution of confusion, and then calculating the interface coefficients.

\subsection{The Symmetric Interface} \label{ss:A}
We begin with an analysis of the one-dimensional effect parameter found in $\eqref{eq:epsilon}$, to identify the implicit assumption about confusion. By setting $\eps_0 = \eps_1 = \hat{\eps}= r_{11}-r_{01}$ in \eqref{eq:epsilonandsigma0} and \eqref{eq:epsilonandsigma1}, we find that
\begin{equation}
(\sigma_0,\sigma_1) = \left( \frac{r_{01}}{r_{01}+r_{10}}, \frac{r_{10}}{r_{01}+r_{10}} \right).
\end{equation}
Adopting this model implies the belief that future interface errors will be distributed like the anti-causal observations counted in $p_{01}$ and $p_{10}$.

There are other interesting relationships to this coefficient, including $\hat{\eps} = \det R$, the matrix determinant, which in this case is also equal to $r_{00} + r_{11} - 1 = \trace{(R)}-1$. Moreover, the two eigenvalues of $R$ are 1 and $\hat{\eps}$.

Taking the specific example we began with, we can understand the example \[ P_1 = \begin{bmatrix}
  .23 & .25 \\
.20  &  .32
 \end{bmatrix} \] as a row-normalized decomposition
 \[\begin{bmatrix}
  .535 & .465 \\
.439 &  .561
 \end{bmatrix}  =
 \begin{bmatrix}
  .904 & 0 \\
0  &  .904
 \end{bmatrix}\begin{bmatrix}
  .485 & .515 \\
.485  &  .515
 \end{bmatrix} +
 \begin{bmatrix}
  .096 & 0 \\
0  &  .096
 \end{bmatrix},
 \]
  with $\hat{\eps} \approx .096.$

Next we will try out a new idea for deriving the corresponding interface coefficients.

\subsection{Maximum Cause} \label{ss:M}
The total explanatory power of the interface is $\eps_0+\eps_1$, so it is of interest to see what the confusion distribution looks like if we maximize this sum. To find it, we can use \eqref{eq:newepsilon0} to find where the slope of the $\eps_1 = f(\eps_0)$ curve is $-1$. Differentiating gives
\[
\frac{d}{d\eps_0} \left( r_{11} - \frac{r_{01} r_{10}}{r_{00}-\eps_0} \right) =
- \left( \frac{r_{01} r_{10}}{r_{00}-\eps_0} \right) ^2,
\]
so to maximize the sum we want
\begin{align}
\eps_0 =& r_{00} - \sqrt{r_{01}r_{10}} \\
\eps_1 =& r_{11} - \sqrt{r_{01}r_{10}}.
\end{align}
Substituting this into \eqref{eq:epsilonandsigma0} and \eqref{eq:epsilonandsigma1} and solving gives
\begin{align}
\sigma_0 =& \frac{\sqrt{r_{10}}}{\sqrt{r_{01}}+\sqrt{r_{10}}} \\
\sigma_1 =& \frac{\sqrt{r_{01}}}{\sqrt{r_{01}}+\sqrt{r_{10}}},
\end{align}
which looks like a variation of the symmetric $\hat{\eps}$ interface we just saw. Both of them rely on the maximum size of potential anti-causes as the basis to calculate the confusion matrix, which is what limits the size of the interface coefficients.

\subsection{Classification Confusion} \label{ss:C}

It may be that the distribution of $B$ is fixed, and for that reason we insist that
$(\sigma_0, \sigma_1) = (p_{*0},p_{*1}).$ This would be the case for a classification problem, for example, where $A$ is not \emph{causing} $B$, but only used as an indicator function to identify it. We might use this as an exploratory tool to look for causal relationships in survey data.  In this case, knowing the row-normed matrix is not enough: we have to have the original distribution of $B$ as well.

The interface coefficients are
\begin{align}
\eps_0 =& 1- \frac{r_{01}}{p_{*1}} \\
\eps_1 =& 1- \frac{r_{10}}{p_{*0}}
\end{align}.

These coefficients can be understood as the relative difference between the original matrix and the product of its margins. This is easier to see in the form
\[ \eps_1 = \frac{p_{1*}p_{*0}-p_{10}}{p_{1*}p_{*0}},
\]
which can also be written
\begin{align}
&= \frac{(p_{10}+p_{11})(p_{00}+p_{10})-p_{10}}{p_{1*}p_{*0}} \\
&= \frac{p_{10}p_{00}+p_{10}^2+ p_{11}p_{00}+p_{11}p_{10}-p_{10}}{p_{1*}p_{*0}} \\
&= \frac{p_{10}(p_{00}+p_{10}+p_{11}-1)+p_{11}p_{00}}{p_{1*}p_{*0}} \\
&= \frac{p_{11}p_{00}-p_{10}p_{01}}{p_{1*}p_{*0}} \\
&= \frac{Cov(A,B)}{p_{1*}p_{*0}}.
\end{align}
Similarly,
\[ \eps_0 = \frac{Cov(A,B)}{p_{0*}p_{*1}}.
\]
A nice feature of this epistemology is that the geometric average of interface coefficients is the correlation coefficient, because
\begin{align}
\eps_0 \eps_1 &= \frac{Cov(A,B)Cov(A,B)}{(p_{0*}p_{1*})(p_{*0}p_{*1})} \\
&= \frac{Cov(A,B)^2}{Var(A)Var(B)}.
\end{align}
 There is also a relationship to the Symmetric Interface $\hat{\eps}$ in that if we compute it for both $R$ and its transpose, the geometric average of these is also the correlation coefficient.

For our example $ \begin{bmatrix}
  .23 & .25 \\
.20  &  .32
 \end{bmatrix}$, we decompose the resulting $R$ as
 \[\begin{bmatrix}
  .535 & .465 \\
.439 &  .561
 \end{bmatrix}  =
 \begin{bmatrix}
.894  & 0 \\
0  & .914
 \end{bmatrix}\begin{bmatrix}
  .48 & .52 \\
.48  &  .52
 \end{bmatrix} +
 \begin{bmatrix}
  .106 & 0 \\
0  &  .086
 \end{bmatrix}.
 \]

\subsection{Untreated Confusion} \label{ss:U}
Some may object to the Classification Confusion, because they want to single out the ``treatment'' case $A=1$ as a new effect, and only use untreated observations coming from the $A=0$ observations as a basis for confusion.
This is the maximum value $\eps_1$ can be, and it gives us a philosophical justification for the $y$-intercept on the interface curve. In this case we have
\begin{equation}
  \begin{bmatrix}
  r_{00} & r_{01} \\
  r_{10} & r_{11}
 \end{bmatrix}
=
\begin{bmatrix}
1-  \eps_0 & 0 \\
  0 &  1- \eps_1
 \end{bmatrix}
 \begin{bmatrix}
  r_{00} & r_{01} \\
  r_{00} & r_{01}
 \end{bmatrix}
+
\begin{bmatrix}
 \eps_0 & 0 \\
  0 &  \eps_1
 \end{bmatrix},
\end{equation}
which gives us
\begin{align}
\eps_0 =& 1- \frac{r_{01}}{r_{01}} \\
=& 0 \\
\eps_1 =& 1 - \frac{r_{10}}{r_{00}},
\end{align}
which is the maximum value possible.

\subsection{Natural Confusion} \label{ss:N}
Recalling that the determinant-zero condition on the confusion matrix in \eqref{eq:detzero} is
\[ (p_{00} - \eps_0 p_{0*})(p_{11} - \eps_1 p_{1*}) =  p_{01} p_{10}, \]
it is natural to assign
\begin{align}
p_{00} - \eps_0 p_{0*} =& p_{01} \\
p_{11} - \eps_1 p_{1*} =& p_{10},
\end{align}
so that
\begin{align}
\eps_0 = & \frac{p_{00} - p_{01}}{p_{0*}} \\
=& r_{00} - r_{01} \\
\eps_1 =& r_{11} - r_{10}
\end{align}
with resulting confusion
\begin{align}
\sigma_0 = & \frac{r_{10}}{1 - r_{11} + r_{10}} \\
         =& \frac{r_{10}}{r_{10} + r_{10}} \\
         =& 1/2 \\
\sigma_1 =& 1/2.
\end{align}

If we chose the other natural factorization (subtracting columns instead of rows), we get
\begin{align}
p_{00} - \eps_0 p_{0*} =& p_{10} \\
p_{11} - \eps_1 p_{1*} =& p_{01},
\end{align}
so that
\begin{align}
\eps_0 =& r_{00} - r_{10} \\
\eps_1 =& r_{11} - r_{01},
\end{align}
which we recognize as our starting point: the Symmetric Interface.

\subsection{Graphical Comparison}
We now have a catalog of a few possible epistemologies, and it is interesting to see how they compare for different types of $P$ matrices. For each of the graphs below, the labels on points are indexed by:
\begin{description}
  \item[S] Symmetric Interface from Section \ref{ss:A}
  \item[M] Maximum Cause from Section \ref{ss:M}
  \item[C] Classification Confusion from Section \ref{ss:C}
  \item[U] Untreated Confusion from Section \ref{ss:U}
  \item[N] Natural Confusion from section \ref{ss:N}
\end{description}

\begin{figure}[H]
  \centering
  \includegraphics[width=0.6\textwidth]{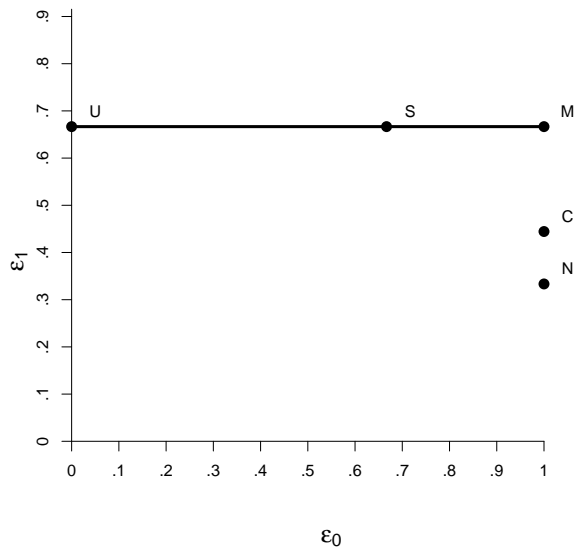}\\
  \caption[Another graph]{Interfaces for $\begin{bmatrix}.4 & 0 \\ .2 & .4 \end{bmatrix}$} \label{Interface example 1}
\end{figure}

\begin{figure}[H]
  \centering
  \includegraphics[width=0.6\textwidth]{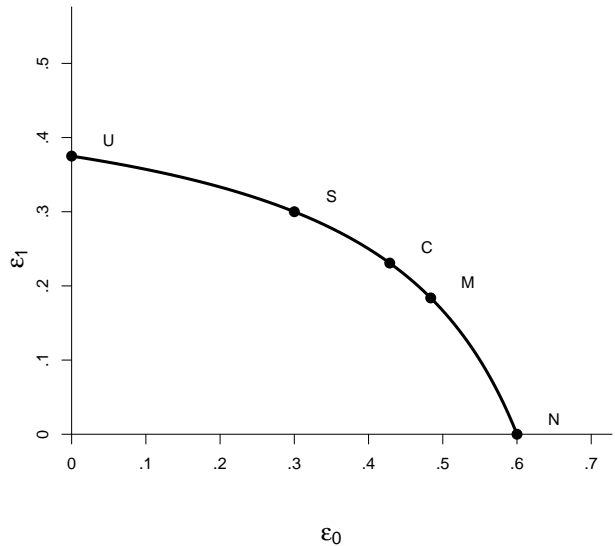}\\
  \caption[Another graph]{Interfaces for $\begin{bmatrix}.40 & .10 \\ .25 & .25 \end{bmatrix}$} \label{Interface example 2}
\end{figure}

\section{Final Remarks}
The preceding discussion shows that considering the respective probabilities of $A \Leftrightarrow B$ has a richness than cannot be captured with a single index of $A \Rightarrow B$ causality. For example, we saw in Section \ref{ss:C} that the correlation coefficient on two binary variables can be seen as an average causal index, and that we can decompose it, given certain assumptions, into a forward and reverse implication with associated probabilities. We also showed how use of a usual measure of effect, $\hat{\eps}$, makes implicit assumptions about the distribution of non-causal results, and is agnostic about the direction of logical implication.

In practice, the complex systems we rely on usually have logic that resembles an interface. We depend on causalities of the form [driver brakes sharply] $\Leftrightarrow$ [vehicle stops suddenly], and this becomes a simplified operating assumption. Each direction of the implication is of interest: the forward direction is a problem in designing good equipment, and the reverse direction amounts to preventing other causes of sudden stops, such as running into a light pole. Both of these activities are important for intelligent agents.

\bibliographystyle{plain}
\bibliography{CausalInterfaces}

\end{document}